\documentclass[journal]{IEEEtran}

\ifCLASSOPTIONcompsoc
  \usepackage[nocompress]{cite}
\else
  \usepackage{cite}
\fi

\ifCLASSINFOpdf
\else
\fi
\usepackage{cite}
\usepackage{times}
\usepackage{epsfig}
\usepackage{graphicx}
\usepackage{amsmath}
\usepackage{amssymb}
\usepackage{enumerate}
\usepackage{booktabs}

\usepackage{enumerate}
\usepackage{algorithm}
\usepackage{algorithmic}
\usepackage{multirow}
\usepackage{amsthm}
\usepackage{color}
\usepackage{colortbl}
\newtheorem{theorem}{Theorem}
\newtheorem{lemma}{Lemma}
\theoremstyle{remark}

\usepackage{bm}
\usepackage{mathtools}
\DeclarePairedDelimiter\ceil{\lceil}{\rceil}

\begin{document}
%
\title{Self-Ensembling GAN for Cross-Domain Semantic Segmentation}

\author{Yonghao~Xu,~\IEEEmembership{Member,~IEEE},~ Fengxiang~He,~\IEEEmembership{Member,~IEEE},~Bo~Du,~\IEEEmembership{Senior~Member,~IEEE},\\ Dacheng~Tao,~\IEEEmembership{Fellow,~IEEE},~and~Liangpei~Zhang,~\IEEEmembership{Fellow,~IEEE} \thanks{Yonghao Xu and Fengxiang He contributed equally to this article. This work was supported in part by the National Natural Science Foundation of China under Grants 41871243, 62225113, 41820104006, 61871299, the Major Science and Technology Innovation 2030 key projects ``New Generation Artificial Intelligence'' (No. 2021ZD0111700), National Key Research and Development Program of China under Grant 2018AAA0101100, and the Science and Technology Major Project of Hubei Province (Next-Generation AI Technologies) under Grant 2019AEA170. \textit{(Corresponding authors: Liangpei Zhang; Bo Du.)}}
\thanks{Y. Xu is with the State Key Laboratory of Information Engineering in Surveying, Mapping, and Remote Sensing, Wuhan University, Wuhan 430079, China, and also with the Institute of Advanced Research in Artificial Intelligence (IARAI), 1030 Vienna, Austria (e-mail: yonghaoxu@ieee.org).}
\thanks{F. He and D. Tao are with JD Explore Academy, JD.com Inc., Beijing 100176, China (e-mail: fengxiang.f.he@gmail.com; dacheng.tao@gmail.com).}
\thanks{B. Du is with School of Computer Science, National Engineering Research Center for Multimedia Software, Institute of Artificial Intelligence, and Hubei Key Laboratory of Multimedia and Network Communication Engineering, Wuhan University, Wuhan 430079, China. (e-mail: dubo@whu.edu.cn).}
\thanks{L. Zhang is with the State Key Laboratory of Information Engineering in Surveying, Mapping, and Remote Sensing, Wuhan University, Wuhan 430079, China (e-mail: zlp62@whu.edu.cn).}
}

\markboth{IEEE Transactions on Multimedia,~Preprint, December~2022}%
{Xu \MakeLowercase{\textit{et al.}}: Self-Ensembling GAN for Cross-Domain Semantic Segmentation}

\IEEEtitleabstractindextext{%
\begin{abstract}
Deep neural networks (DNNs) have greatly contributed to the performance gains in semantic segmentation. Nevertheless, training DNNs generally requires large amounts of pixel-level labeled data, which is expensive and time-consuming to collect in practice. To mitigate the annotation burden, this paper proposes a self-ensembling generative adversarial network (SE-GAN) exploiting cross-domain data for semantic segmentation. In SE-GAN, a teacher network and a student network constitute a self-ensembling model for generating semantic segmentation maps, which together with a discriminator, forms a GAN. Despite its simplicity, we find SE-GAN can significantly boost the performance of adversarial training and enhance the stability of the model, the latter of which is a common barrier shared by most adversarial training-based methods. We theoretically analyze SE-GAN and provide an $\mathcal O(1/\sqrt{N})$ generalization bound ($N$ is the training sample size), which suggests controlling the discriminator's hypothesis complexity to enhance the generalizability. Accordingly, we choose a simple network as the discriminator. Extensive and systematic experiments in two standard settings demonstrate that the proposed method significantly outperforms current state-of-the-art approaches. The source code of our model is available online (https://github.com/YonghaoXu/SE-GAN).
\end{abstract}

\begin{IEEEkeywords}
Deep learning, domain adaptation, semantic segmentation, adversarial learning.
\end{IEEEkeywords}}

\maketitle

\IEEEdisplaynontitleabstractindextext

%
\IEEEpeerreviewmaketitle

\IEEEpeerreviewmaketitle

\section{Introduction}\label{sec:introduction}

%
%
%
%

\IEEEPARstart{D}{eep} neural networks (DNNs) have played a central role in semantic segmentation \cite{zhang2019decoupled,kang2018depth,chen2018importance}. Nevertheless, training DNNs usually requires large amounts of pixel-level annotated data, which is expensive and time-consuming to collect in real-world scenarios \cite{zhou2020sal,pan2022learning,fang2022incremental,zhang2020few}. A promising way to mitigate the annotation burden is to use synthetic images generated from cross-domain data in training the models \cite{wang2022cross,guan2021uncertainty,zhang2019curriculum}. However, these synthetic images usually have domain shifts from the real images \cite{domainshift}, which may severely undermine the performance, since supervised learning generally assumes that the training and test data share the same distribution \cite{mohri2018foundations}. Fig. \ref{fig:domain_shift} shows two major types of domain shifts in semantic segmentation, visual appearance shift and label distribution shift. Specifically, visual appearance shift refers to the shift in ``image styles" caused by differences in illumination, distortion, and imaging devices in different domains; and label distribution shift is defined as the difference in the marginal distribution of the spatial layout in the label space (for example, the spatial layout of objects and scenes may vary in different cities).

\begin{figure}
\centering
\includegraphics[width=\linewidth]{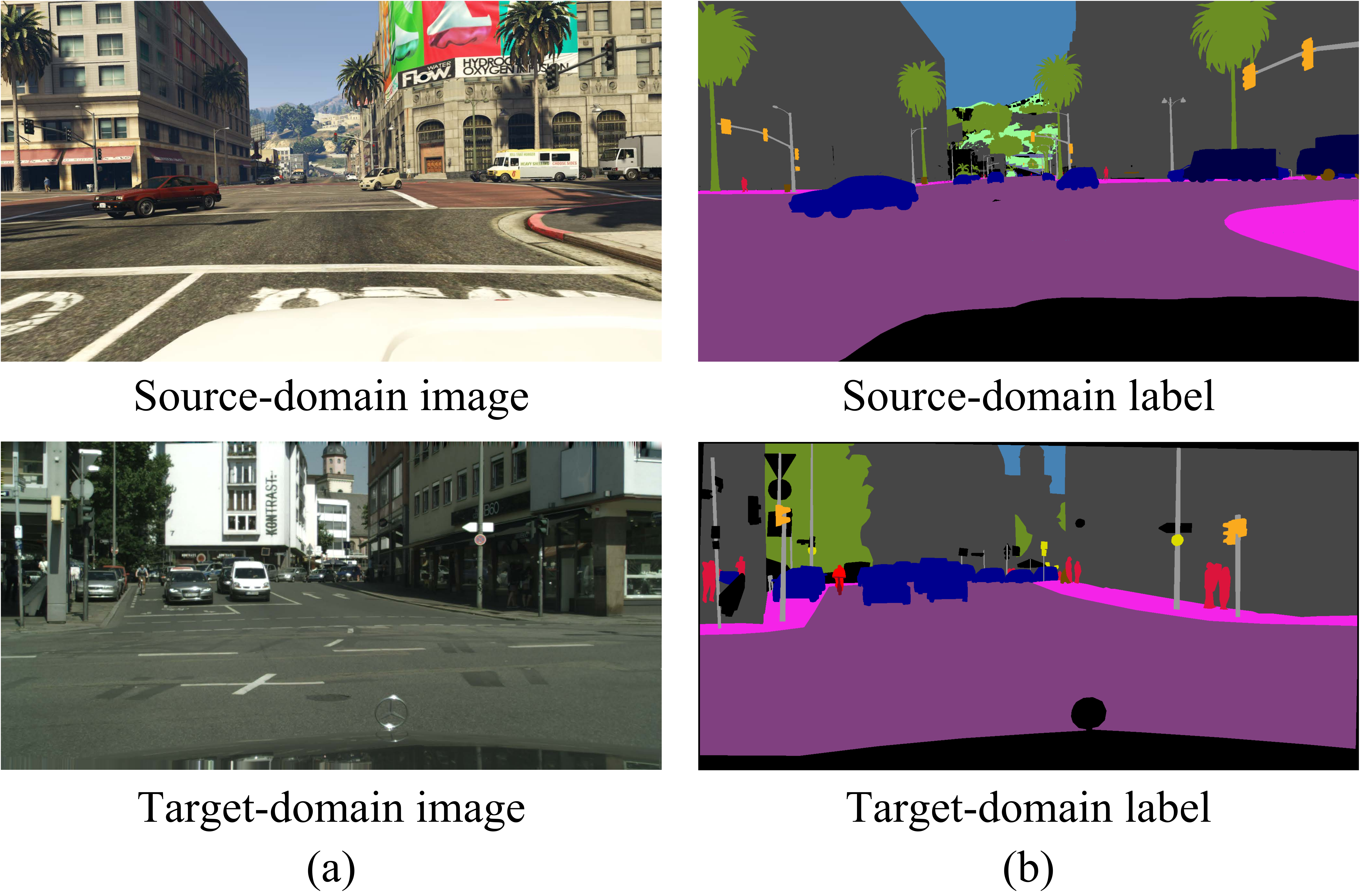}
\caption{An illustration of the domain shift phenomenon in semantic segmentation. Examples from the GTA-5 and Cityscapes datasets are selected as the source and target domains, respectively. (a) Visual appearance shift: images from different domains have distinctive visual styles. (b) Label distribution shift: the spatial layout in the label space may vary from different domains.}
\label{fig:domain_shift}
\end{figure}

To address this issue, domain adaptation techniques have been deployed to transfer knowledge between different domains \cite{song2020learning,chen2019learning}. Of these techniques, adversarial training-based methods have shown their potential \cite{adaptseg,luo2019taking,luo2019significance,tsai2019domain}.
One of the representative works is done by Tsai \textit{et al.}, where they proposed a novel framework to conduct domain adaptation in the output space with adversarial training \cite{adaptseg}. By alternately training the generator (which is implemented with a segmentation network) and the discriminator, the model is expected to learn domain-invariant segmentation results in the output space. This approach is later extended by Luo \textit{et al.}, where the category-level adversarial training is proposed to achieve fine-grained domain adaptation \cite{luo2019taking}.

\begin{figure}
\centering
\includegraphics[width=\linewidth]{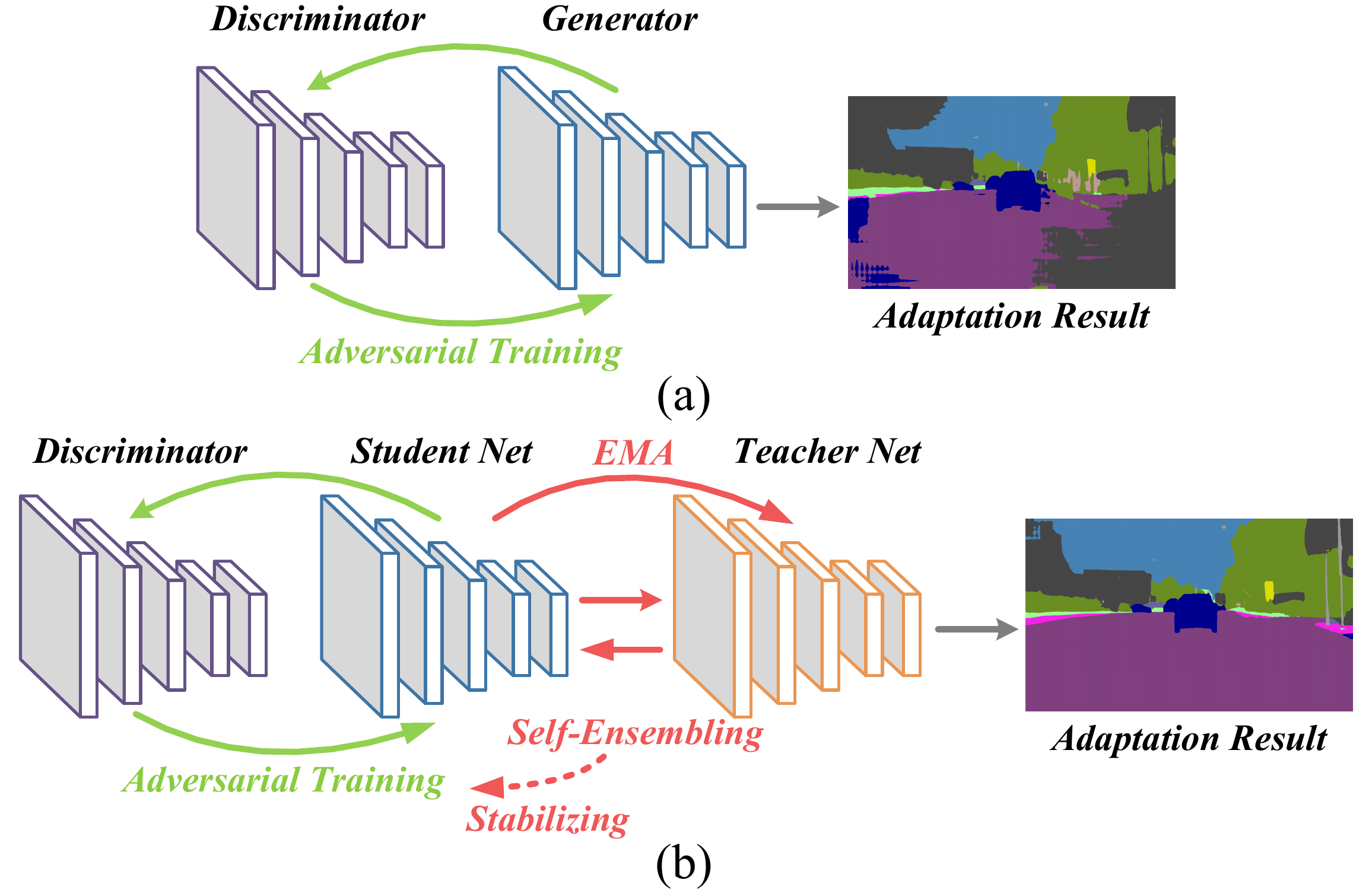}
\caption{An intuitional comparison between (a) the traditional adversarial training-based method and (b) the proposed SE-GAN. ``EMA'' denotes the exponential moving average. While traditional adversarial training adopts a single base model to act as the generator, we use a teacher network and student network to constitute a self-ensembling model which achieves a more stable generator and thereby stabilizes the adversarial training (see Fig. 6 for more details).}
\label{fig:adv_training}
\end{figure}

One disadvantage of the aforementioned methods is the instability of training. Since the adversarial training-based framework involves minimax optimization by alternately training the generator and discriminator, it is common to observe that the domain adaptation performance could fluctuate sharply between adjacent iterations and the training procedure is sometimes very fragile\footnote{https://github.com/wasidennis/AdaptSegNet/issues/29}.
This phenomenon may further lead to the risk of negative transfer, where the learned features could generalize badly in the target domain due to the unstable training of adversarial networks \cite{wu2018dcan}, resulting in poor performance.

This paper proposes a novel self-ensembling generative adversarial network (SE-GAN) for cross-domain semantic segmentation. Different from existing adversarial training-based approaches where only a single segmentation network is adopted to play the role of generator \cite{adaptseg,luo2019taking}, we propose to use the ensemble model to achieve a more stable generator. Specifically, we use a teacher network and a student network to constitute a self-ensembling model for generating cross-domain semantic segmentation maps, which together with a discriminator, forms a GAN. Both teacher and student networks share the same architecture. During the training phase, the student network is optimized through backpropagation, while the parameters of the teacher network are updated from the historical parameters of the student network in recent iterations. In this way, the teacher network unites many student networks to form an ensemble model. Since an ensemble model generally yields better predictions and possesses stronger robustness towards noises than a single base model (please refer to Fig. 1 in \cite{tarvainen2017mean} for more detailed discussions), it may also help to achieve a more stable generator and thereby stabilize the adversarial training. An intuitional comparison between the traditional adversarial training-based method and the proposed SE-GAN is presented in Fig. \ref{fig:adv_training}.

To summarize, the main contribution of this study is the proposed SE-GAN, which entangles two promising yet distinct schemes, adversarial training and self-ensembling. It inherits advantages from both regimes and addresses each other’s major critical shortcomings. Besides, we evaluate the proposed method from both theoretical and empirical perspectives. First, we theoretically analyze the proposed method and give an $\mathcal O\left(1/\sqrt{N}\right)$ generalization bound for SE-GAN, where $N$ is the size of the training set. This generalization bound suggests that when $N$ approaches infinity, the generalization gap between the target distribution and the empirical distribution of the generated data approaches $0$. Also, these generalization bounds suggest controlling the discriminator's hypothesis complexity for good generalization. Therefore, we employ a simple network as the discriminator. Second, extensive and systematic experiments on benchmark datasets suggest that our proposed method can significantly stabilize adversarial training and outperform current state-of-the-art approaches.

\section{Related Work}

\textbf{Semantic segmentation} is of fundamental importance in many computer vision tasks, which aims to assign a category label to each pixel in a given image \cite{fcn}. Representative works include SegNet \cite{segnet}, U-Net \cite{unet}, DeepLab \cite{deeplab}, PSPNet \cite{zhao2017pyramid}, and RefineNet \cite{lin2017refinenet}. Although these approaches perform well, their success is largely based on large amounts of pixel-level annotated data, which is very expensive and time-consuming to collect. To mitigate the annotation burden, Richter \textit{et al.} explored the use of realistic computer games for creating pixel-accurate ground truth data \cite{gta}. Ros \textit{et al.} proposed to use a virtual world to automatically generate realistic synthetic images with pixel-level annotations \cite{synthia}. However, the models directly trained on synthetic data fail to perform well on real-world datasets due to domain shifts \cite{zhang2019category,zhao2019multi,chen2019crdoco}.

\textbf{Adversarial training} has been widely deployed in domain adaptation tasks \cite{ganin2016domain}. Hoffman \textit{et al.} proposed the first adversarial training-based approach to tackle cross-domain semantic segmentation \cite{fcnswild} and considered both global and local adaptation techniques to learn domain-invariant features. Zhang \textit{et al.} proposed a novel fully-convolutional adaptation network for semantic segmentation, which combined appearance adaptation networks and representation adaptation networks \cite{zhang2018fully}. Tsai \textit{et al.} proposed a novel AdaptSeg framework to conduct domain adaptation in the output space to address the label distribution shift \cite{adaptseg}. Luo \textit{et al.} extended AdaptSeg with the category-level adversarial training to achieve fine-grained domain adaptation \cite{luo2019taking}. Vu \textit{et al.} proposed an entropy-based adversarial training approach to penalize low-confident predictions on target domain \cite{vu2019advent}. Pan \textit{et al.} further use adversarial training to conduct intra-domain adaptation \cite{pan2020unsupervised}. One disadvantage of these approaches is the risk of negative transfer, where the learned features may generalize badly in the target domain due to the unstable training of adversarial networks \cite{wu2018dcan}.

\begin{figure*}
\centering
\includegraphics[width=0.98\linewidth]{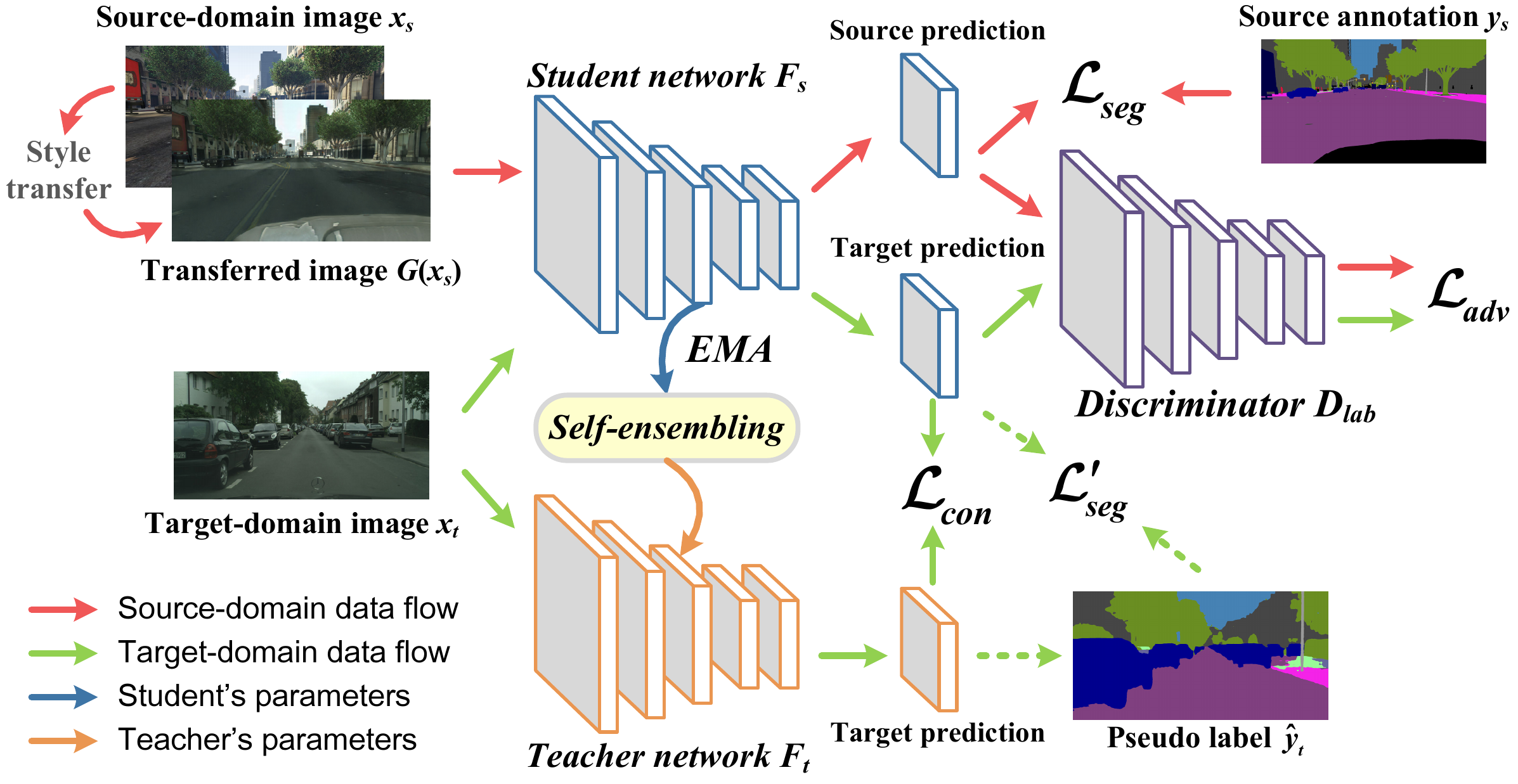}
\caption{An overview of the proposed self-ensembling generative adversarial network (SE-GAN) for cross-domain semantic segmentation. 
We optimize the student network with the combination of cross-entropy loss, adversarial loss, and consistency loss by the backpropagation algorithm. The parameters in the teacher network are updated with the exponential moving average (EMA) of the historical parameters in the student network at different iterations.
}
\label{fig:SEGAN}
\end{figure*}

\textbf{Self-ensembling networks} have demonstrated potentials in semi-supervised learning and cross-domain tasks, which originate from the $\Pi$-model and temporal model \cite{laine2016temporal,yao2020deep}. Tarvainen \textit{et al.} further designed a mean teacher model which achieved the state-of-the-art performance in semi-supervised classification tasks \cite{tarvainen2017mean}. French \textit{et al.} employed the mean teacher model to solve cross-domain image classification problem \cite{french2018self}. Xu \textit{et al.} further introduced self-ensembling to cross-domain semantic segmentation task  \cite{xu2019self}. Choi \textit{et al.} proposed to employ a style transfer network for data augmentation in self-ensembling model \cite{choi2019self}. Since there are no specific object functions in existing self-ensembling networks to address domain shifts, directly using self-ensembling to implement cross-domain semantic segmentation may bring about limited improvement.

In contrast to existing related methods, we adopt the self-ensembling model to play the role of a generator in the proposed SE-GAN, which inherits the stability of ensemble learning and can significantly improve the performance of adversarial training in cross-domain semantic segmentation.

\section{Methods}

\subsection{Overview of the Proposed Model}
SE-GAN proposes to adopt the self-ensembling model to act as the generator in adversarial training for domain adaptation. As shown in Fig. \ref{fig:SEGAN}, SE-GAN has three main components: a student network, a teacher network, and a discriminator. Both the student and teacher networks in SE-GAN share the same architecture. The student and teacher networks also play the role of a generator that aims to generate domain-invariant segmentation results in the label space. The discriminator, on the other hand, aims to distinguish whether the segmentation maps come from the source domain or the target domain. For each input source-domain image, we first use a style-transfer network (see Appendix A-C for details) to transfer its visual style to be consistent with that of the target domain, while simultaneously retaining its semantic information. This can help to address the visual appearance shift \cite{hoffman2018cycada}.

In the training phase, we optimize the student network with the weighted sum of cross-entropy loss, adversarial loss, and consistency loss by backpropagation. The parameters in the teacher network are updated with the exponential moving average (EMA) of the historical parameters in the student network at different iterations to achieve self-ensembling. Once the training is finished, we further use the teacher network to generate pseudo labels for unlabeled target-domain images and conduct self-training for the student network on the target domain.

\subsection{Self-ensembling Learning}
Due to factors such as illumination, distortion, and different imaging devices, images from different domains usually have distinctive visual styles, resulting in a visual appearance shift. An intuitive idea to address this issue is to transfer the style of the images from the source domain to be consistent with that of the target domain and then train the semantic segmentation model with both original source domain images and transferred images \cite{hoffman2018cycada}.

Formally, let $S$ and $T$ represent the source domain and target domain, respectively. Given a set of source-domain images $X_{S}$, the corresponding labels $Y_{S}$, and target-domain images $X_{T}$ (without annotations), we first use a style transfer network $G$ (see Appendix A-C for details) to transfer the visual style of $X_{S}$ in source domain $S$ to be consistent with that of target domain $T$ while retaining the semantic layout of image $X_{S}$. In this way, each source-domain image $x_{s}\in X_{S}$ and the corresponding transferred image $G\left(x_{s}\right)$ can share the same annotation $y_s$. Let $F_s$ and $F_t$ denote the student network and teacher network, respectively. The cross-entropy loss $\mathcal L_{seg}$ for the student network can therefore be defined as:
\begin{align}
    \mathcal{L}_{seg}\left(F_s\right)=-\frac{1}{2K}\sum_{k=1}^{K}\sum_{c=1}^{C}&y_s^{\left(k,c\right)}
    \left[\log \left(\sigma\left(F_s\left(x_s\right)\right)^{\left(k,c\right)}\right)\right. \nonumber\\
    &\left.+\log \left(\sigma\left(F_s\left(G\left(x_s\right)\right)\right)^{\left(k,c\right)}\right)\right],
\label{eq:1}
\end{align}
where $K$ and $C$ denote the number of pixels in the image, and number of categories in the segmentation task, respectively. $F_s\left(\cdot\right)$ denotes the segmentation output of the student network. $\sigma$ denotes the softmax function. Since the cross-entropy loss $\mathcal L_{seg}$ is calculated with both $x_s$ and $G(x_s)$, the visual appearance shift between the source domain $S$ and the target domain $T$ can thereby be reduced.

In order to achieve the self-ensembling mechanism, we then feed each target-domain image $x_{t}\in X_{T}$ into both the student network $F_s$ and the teacher network $F_t$. The consistency loss $\mathcal L_{con}$ can be formulated as:
\begin{equation}
    \mathcal{L}_{con}\left(F_s\right)=\frac{1}{K}\sum_{k=1}^{K}\sum_{c=1}^{C}\left\|\sigma\left(F_s\left( x_t\right)\right)^{\left(k,c\right)}
    -\sigma\left(F_t\left(x_t\right)\right)^{\left(k,c\right)}\right\|^2,
\label{eq:2}
\end{equation}
where $F_t\left(\cdot\right)$ denotes the segmentation output of the teacher network. With the constraint in (\ref{eq:2}), the student network $F_s$ can gradually learn from the teacher network $F_t$ on the target domain $T$.

\subsection{Output Space Domain Adaptation with Ensemble Learning}
Recall that our goal is to learn a domain-invariant feature representation in the label space. Inspired by the work in \cite{adaptseg}, we define a discriminator $D_{lab}$ which aims to distinguish whether the input segmentation map comes from the source domain $S$ or the target domain $T$. Accordingly, the student network $F_s$ acts as the generator that aims to fool $D_{lab}$. The corresponding adversarial loss $\mathcal{L}_{adv}$ can be formulated as:
\begin{align}
\mathcal{L}_{adv}\left(F_s,D_{lab}\right)&= \mathbb{E}_{x_s\sim X_S}\left[\log \left(1-D_{lab}\left(F_s\left(x_s\right)\right)\right)\right.\nonumber\\
&+\left.\log \left(1-D_{lab}\left(F_s\left(G\left(x_s\right)\right)\right)\right)\right] \nonumber\\
& + \mathbb{E}_{x_t\sim X_T}\left[\log D_{lab}\left(F_s\left(x_t\right)\right)\right].
\label{eq:3}
\end{align}

Note that we regard the label of $F_s\left(G\left(x_s\right)\right)$ as the source domain, since style transfer does not change the spatial layout of $x_s$ in the label space. Similar to previous adversarial training based methods \cite{adaptseg}, we optimize (\ref{eq:3}) through a minimax two-player game. Specifically, we alternately optimize $\mathop{\min}_{F_s}\mathcal{L}_{adv}$ and $\mathop{\max}_{D_{lab}}\mathcal{L}_{adv}$.

The full objective function for training the proposed SE-GAN can be formulated as:
\begin{equation}
\mathop{\min}_{F_s}\mathop{\max}_{D_{lab}}\mathcal{L}_{seg}\left(F_s\right)+\lambda_{con}\mathcal{L}_{con}\left(F_s\right)+\lambda_{adv}\mathcal{L}_{adv}\left(F_s,D_{lab}\right),
\label{eq:4}
\end{equation}
where $\lambda_{con}$ and $\lambda_{adv}$ are two weighting factors.

\begin{algorithm}
\caption{Training SE-GAN}
    \label{alg:segan}
    {\bf Input:}
\begin{enumerate}[-]
\item The labeled source-domain training set $X_{S}$ and $Y_{S}$, and the unlabeled target-domain training set $X_{T}$.
\item A student network $F_s$, a teacher network $F_t$, a trained style transfer network $G$, and a discriminator $D_{lab}$.
\item The hyper-parameters $\lambda_{con}$, $\lambda_{adv}$, and $\alpha$.
\end{enumerate}
 \begin{algorithmic}[1]
    \STATE Initialize $F_s$ and $F_t$ with the ImageNet pre-trained parameters; Initialize $D_{lab}$ with the random Gaussian values.
    \FOR{$t$ in $range\left(0,maxiter\right)$}
    \STATE Get $x_s\in X_S$, $y_s\in Y_S$, $x_t\in X_T$.
    \STATE Transfer $x_s$ into the target-domain style by $G\left(x_s\right)$.
    \STATE Compute the cross-entropy loss $\mathcal L_{seg}\left(F_s\right)$, the consistency loss $\mathcal L_{con}\left(F_s\right)$, and the adversarial loss $\mathcal{L}_{adv}\left(F_s,D_{lab}\right)$ via \eqref{eq:1}, \eqref{eq:2}, and \eqref{eq:3}.
    \STATE Update $F_s$ by descending its stochastic gradient via $\nabla_{F_s}\left(\mathcal{L}_{seg}+\lambda_{con}\mathcal{L}_{con}+\lambda_{adv}\mathcal{L}_{adv}\right).$
    \STATE Update $D_{lab}$ by ascending its stochastic gradient via $\nabla_{D_{lab}}\mathcal{L}_{adv}.$
    \STATE Update $F_t$ by the EMA via \eqref{eq:5}.
    \ENDFOR
    \FOR{$t$ in $range\left(0,maxiter\right)$}
    \STATE Get $x_t\in X_T$.
    \STATE Generate $x_t$'s pseudo label $\hat{y}_t$ by $F_t$.
    \STATE Compute the cross-entropy loss $\mathcal{L}_{seg}^\prime\left(F_s\right)$ via \eqref{eq:6}.
    \STATE Update $F_s$ by descending its stochastic gradient via $\nabla_{F_s}\mathcal{L}_{seg}^\prime.$
    \ENDFOR
\end{algorithmic}
{\bf Output:} The trained student network $F_s$.
\end{algorithm}

Owing to the constraint of the consistency loss $\mathcal L_{con}$ in (\ref{eq:2}), the student network could gradually get regularized by the teacher network, thus showing stronger stability in the adversarial training phase compared to traditional methods (see Fig. \ref{fig:training_plot} for more details).

It should be noticed that the parameters in the teacher network $F_t$ do not participate in backpropagation. Instead, we use the exponential moving average (EMA) of $F_s$ to manually update $F_t$ after each iteration. Let $\theta_s^{i}$ and $\theta_t^{i}$ denote the parameters of $F_s$ and $F_t$ at the $i$th iteration, respectively. Then, $\theta_t^{i}$ can be updated by:
\begin{equation}
    \theta_t^{i}=\alpha\theta_t^{i-1}+\left(1-\alpha\right)\theta_s^{i},
\label{eq:5}
\end{equation}
where $\alpha$ is a smoothing coefficient, and $\theta_t^{0}$ is initialized with the ImageNet pre-trained parameters.

\subsection{Self-training with Teacher's Pseudo Labels}
Pseudo labeling is a commonly used technique in semi-supervised learning \cite{lee2013pseudo} and has recently been introduced into the cross-domain semantic segmentation task \cite{zou2018unsupervised}. Different from previous self-training approaches which may require progressive selections for the most confident pseudo labels \cite{du2019ssf}, we simply conduct self-training with all teacher's pseudo labels on the target-domain images to finetune the student network $F_s$, since the optimized teacher network $F_t$ could already provide stable and high-quality pseudo labels.

Specifically, for each unlabeled image $x_{t}\in X_{T}$, we forward it into the teacher network $F_t$ and get the soft segmentation map $p_t=F_t\left(x_t\right)$. Then, $p_t$ is converted to the corresponding pseudo label $\hat{y}_t$, where each entry is an one-hot vector. The self-training is achieved by minimizing the following cross-entropy loss:
\begin{equation}
    \mathcal{L}_{seg}^\prime\left(F_s\right)=-\frac{1}{K}\sum_{k=1}^{K}\sum_{c=1}^{C}\hat{y}_t^{\left(k,c\right)}\log \left(\sigma\left(F_s\left(x_t\right)\right)^{\left(k,c\right)}\right).
\label{eq:6}
\end{equation}

The detailed training procedure of the proposed SE-GAN is shown in Algorithm~\ref{alg:segan}.

\section{Theoretical Analysis}
This section theoretically evaluates the proposed method. Specifically, we analyze the hypothesis complexity and generalization ability of the proposed SE-GAN. The results theoretically guarantee the performance of our method. A recent work proves a generalization bound that only relies on the hypothesis complexity of the discriminator (see Lemma  \ref{lemma:generalization_bound_gan}
 in Appendix A;
 cf. Theorem 3.1).
This result demonstrates that to achieve good generalization performance, one needs to restrict the complexity of the discriminator. Following this idea, we employ a simple CNN as the discriminator in SE-GAN.

Denote the target distribution and the distribution of generated data $\mu$ and $\nu$, respectively. Suppose the size of training samples is $N$. Denote the empirical distribution of the training set $\hat \mu_{N}$ and the empirical distribution of the generated data $\nu_{N}$. Suppose all the generators $g$ and all the discriminators $f$ form hypothesis classes $\mathcal G$ and $\mathcal F$, respectively.

An adversarial training model is utilized to generate a group of new examples that are identically distributed and follow the same distribution as the existing data. The adversarial training procedure minimizes the gap between the target distribution $\nu$ and the distribution $\hat \mu_{N}$ of the generated data. Mathematically, the adversarial training model minimizes the integral probability metric (IPM) $d_{\mathcal F} (\hat \mu_{N}, \nu)$ between the distributions $\hat \mu_{N}$ and $\nu$ \cite{muller1997integral}, which is defined as:
\begin{equation}
\label{eq:IPM}
    d_{\mathcal F} (\hat \mu_{N}, \nu) \triangleq \sup_{f \in \mathcal F} \left\{ \mathbb E_{x \in \hat \mu_{N}} [f(x)] - \mathbb E_{x \in \nu} [f(x)] \right\}.
\end{equation}

As demonstrated in Fig. \ref{fig:SEGAN}, the employed discriminator in SE-GAN is constituted by a series of convolutional layers and nonlinear activations (nonlinearities). We denote them $(A_{1}, \sigma_{1}, A_{2}, \sigma_{2}, A_{3},\sigma_{3}, A_{4}, \sigma_{4}, A_{5}, \sigma_5)$, where $A_{i}$ is the weight matrix of the $i$-th convolutional layer and $\sigma_{i}$ is a nonlinearity. Then, we can upper bound the hypothesis complexity of the discriminator.

\begin{theorem}[Covering bound for the discriminator]
\label{thm:cover_bound}
    Suppose 
    the spectral norm of each weight matrix is bounded: $\|A_{i}\|_{\sigma} \le s_{i}$. Also, suppose each weight matrix $A_{i}$ has a reference matrix $M_{i}$, which satisfies that $\| A_{i} - M_{i} \|_{\sigma} \le b_{i}$, $i = 1, \ldots, 5$. The Lipschitz constant of $\sigma_5$ is denoted as $\rho$. Then, the $\varepsilon$-covering number satisfies that
    \begin{align}
    \label{eq:cover_bound}
        & \log \mathcal N \left(\mathcal F|_{S}, \varepsilon, \| \cdot \|_{2} \right) \nonumber\\
        \le & \frac{\log\left( 2W^{2} \right) \| X \|_{2}^{2}}{\varepsilon^{2}} \left( \rho \prod_{i = 1}^{5} s_{i} \right)^{2} \left( \sum_{i = 1}^{5} \frac{b_{i}^{2/3}}{s_{i}^{2/3}} \right)^3,
    \end{align}
    where $W$ is the largest dimension of the feature maps throughout the algorithm.
\end{theorem}

A detailed proof is omitted here but provided in the Appendix. Finally, we obtain the following theorem. 
For brevity, we denote the right-hand side (RHS) of (\ref{eq:cover_bound}) $\frac{R^{2}}{\varepsilon^{2}} $.

\begin{figure*}[t]
\centering
\includegraphics[width=0.95\linewidth]{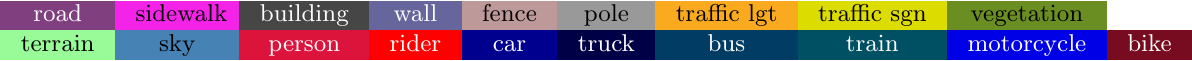}
\includegraphics[width=0.95\linewidth]{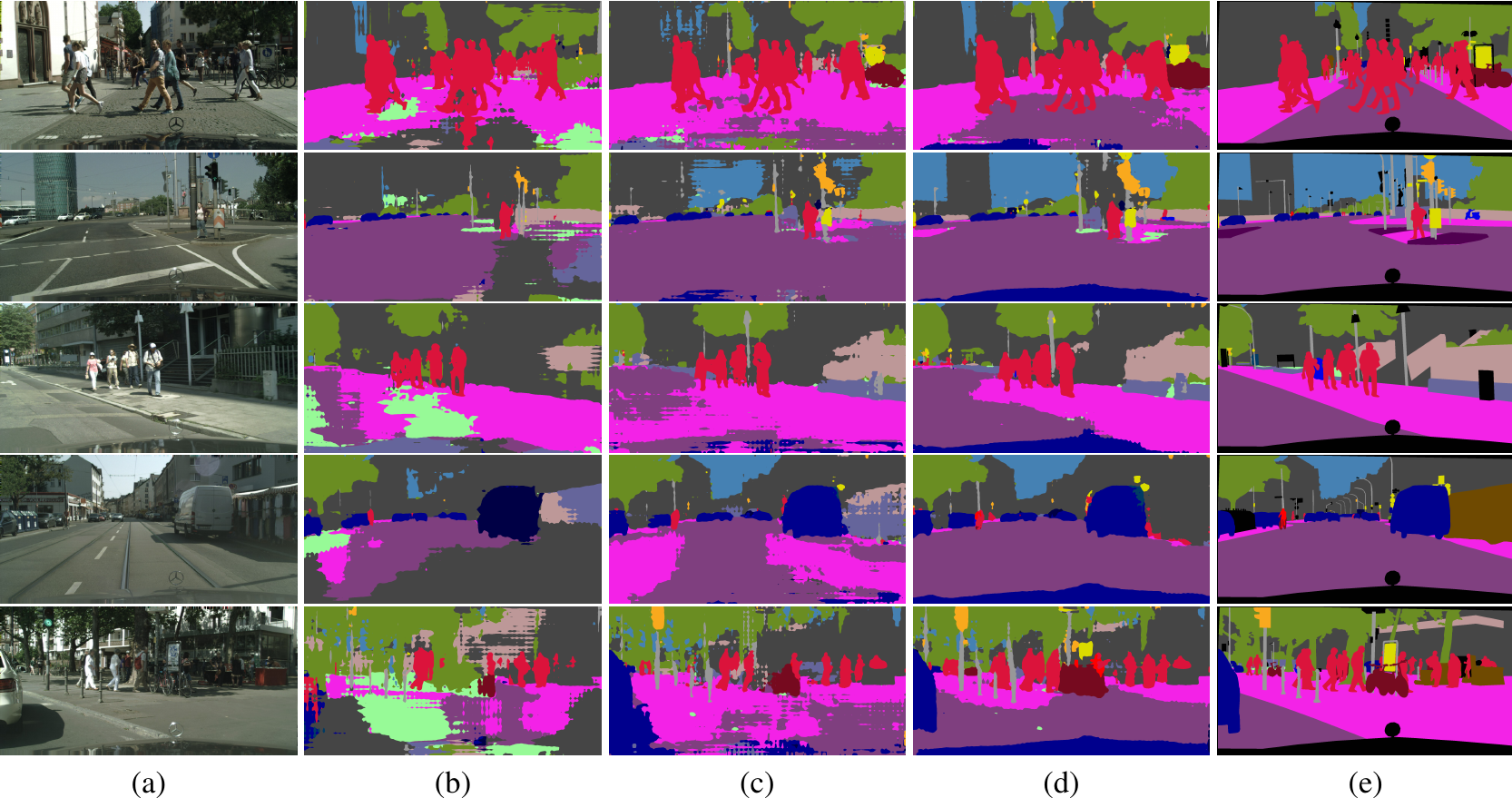}
\caption{Qualitative semantic segmentation results from GTA-5 to Cityscapes datasets with different strategies. (a) Input images from Cityscapes dataset. (b) NoAdapt. (c) Adversarial Training. (d) SE-GAN. (e) Ground-truth annotations.}
\label{fig:segablation}
\end{figure*}

\begin{theorem}
\label{thm:generalization_bound}
    Assume that the discriminator set $F$ is even, i.e., $f \in \mathcal F$ implies $-f \in \mathcal F$ and that all discriminators are bounded by $\Delta$, i.e., $\| f \|_{\infty} \le \Delta$ for any $f \in \mathcal F$. 
    Assume $\hat \mu_{N}$ and $\nu_{N}$ 
    satisfy
    \begin{equation}
        d_{\mathcal F} (\hat \mu_{N}, \nu_{N}) \le \inf_{\nu \in \mathcal G} d_{\mathcal F} (\hat \mu_{N}, \nu) + \phi.
    \end{equation}
    Then, with probability at least $1 - \delta$, we have
    \begin{align}
    \label{eq:generalization_bound}
        & d_{\mathcal F}(\mu, \nu_{N}) - \inf_{\nu \in \mathcal G} d_{\mathcal F}(\mu, \nu) \nonumber\\
        \le & \frac{24R}{N} \left( 1 + \log \frac{N}{3R} \right) + 2 \Delta \sqrt{\frac{2 \log (\frac{1}{\delta})}{N}} + \phi.
    \end{align}
\end{theorem}

Two terms involve the training sample size $N$, $\mathcal O(\log N/N)$ and $\mathcal O(1/\sqrt{N})$. When the training sample size approaches infinity, both terms approach $0$ and the ratio between them approaches $0$. Therefore, this generalization bound is dominated by the second term and the generalization bound is $\mathcal O(1/\sqrt{N})$.

\section{Experiments}

\subsection{Datasets}
We use the Cityscapes \cite{cityscapes} as our target domain data. For the source domain, two challenging synthetic datasets, Synthia \cite{synthia} and GTA-5 \cite{gta} are utilized.

\textbf{Cityscapes} contains real-world vehicle-egocentric images collected from 50 cities in Germany and its surrounding countries. It provides three disjoint subsets: 2975 training images, 500 validation images, and 1525 test images. It also provides accurate pixel-level annotations for all images in 19 different categories. To ensure experimental fairness, we follow the same evaluation protocol as specified in the previous works \cite{pan2020unsupervised}. In the training phase, we use 2975 unlabeled training images, while in the test phase, 500 validation images are used for evaluation.

\textbf{GTA-5} contains 24,966 high-quality labeled frames from the realistic open-world computer games Grand Theft Auto V (GTA-5). Each frame is generated from the fictional city Los Santos, based on Los Angeles in Southern California, with annotations that are compatible with Cityscapes. We use all 19 official training classes for quantitative assessments.

\textbf{Synthia} is a large dataset of photo-realistic frames rendered from a virtual city with precise pixel-level annotations. Following previous works \cite{pan2020unsupervised}, we use the Synthia-Rand-Cityscapes subset, which contains 9400 images with annotations compatible with Cityscapes. The 16 common categories between Synthia and Cityscapes are selected for quantitative assessments.

\subsection{Implementation Details}
We employ DeepLab-v2 \cite{deeplab} with the ResNet-101 \cite{resnet} backbone model pre-trained on ImageNet \cite{imagenet} as the segmentation network for both the student network $F_s$ and the teacher network $F_t$. For the discriminator $D_{lab}$, we use an architecture similar to \cite{radford2015unsupervised} but utilize all fully-convolutional layers to retain the spatial information. The network consists of $5$ convolution layers with kernel $4\times4$ and stride of $2$, where the channel number is $\{64, 128, 256, 512, 1\}$, respectively. Each convolution layer is followed by a Leaky-ReLU \cite{maas2013rectifier} parameterized by $0.2$ except the last layer. The stochastic gradient descent (SGD) optimizer with a learning rate of $2.5e-5$ is utilized to train the student network $F_s$, while the discriminator $D_{lab}$ is optimized by the Adam optimizer \cite{kingma2014adam} with $\beta_1=0.9$, $\beta_2=0.99$. For both optimizers, we set a weight decay of $5e-5$ and adopt the ``poly'' learning rate decay policy, where the initial learning rate is multiplied by $\left(1-iter/maxiter\right)^{power}$ with $power=0.9$. The number of total training iterations $maxiter$ is set to $80000$. $\lambda_{con}$ and $\lambda_{adv}$ in (\ref{eq:4}) are empirically set to $3$ and $0.001$, respectively. The smoothing coefficient $\alpha$ in (\ref{eq:5}) is fixed to $0.999$.

All experiments in this paper were conducted on a computing cluster with a single NVIDIA Tesla V100 GPU.

\begin{table*}[t]
\caption{Quantitative results of semantic segmentation by domain adaptation from GTA-5 to CityScapes. Best results in each column are highlighted in \textbf{bold}. The results of 19 common object categories are reported as mIoU.}
\centering
\resizebox{\textwidth}{!}{
\begin{tabular}{c|ccccccccccccccccccc|c} %
\toprule
\multicolumn{21}{c}{GTA-5$\to$CityScapes}\\
\hline
Methods&\rotatebox{90}{road}&\rotatebox{90}{sidewalk}&\rotatebox{90}{building}&\rotatebox{90}{wall}&\rotatebox{90}{fence}&\rotatebox{90}{pole}&\rotatebox{90}{light}&\rotatebox{90}{sign}&\rotatebox{90}{vegetation}&\rotatebox{90}{terrain}&\rotatebox{90}{sky}&\rotatebox{90}{person}&\rotatebox{90}{rider}&\rotatebox{90}{car}&\rotatebox{90}{truck}&\rotatebox{90}{bus}&\rotatebox{90}{train}&\rotatebox{90}{motocycle}&\rotatebox{90}{bicycle}&\rotatebox{90}{mIoU}\\
\midrule
AdaptSeg \cite{adaptseg}& 86.5&36.0&79.9&23.4&23.3&23.9&35.2&14.8&83.4&33.3&75.6&58.5&27.6&73.7&32.5&35.4&3.9&30.1&28.1&42.4\\
DCAN \cite{wu2018dcan}& 88.5&37.4&79.3&24.8&16.5&21.3&26.3&17.4&80.8&30.9&77.6&50.2&19.2&77.7&21.6&27.1&2.7&14.3&18.1&38.5\\
CLAN \cite{luo2019taking}& 87.0&27.1&79.6&27.3&23.3&28.3&35.5&24.2&83.6&27.4&74.2&58.6&28.0&76.2&33.1&36.7&6.7&31.9&31.4&43.2\\
SWD \cite{lee2019sliced}&
92.0&46.4&82.4&24.8&24.0&\textbf{35.1}&33.4&34.2&83.6&30.4&80.9&56.9&21.9&82.0&24.4&28.7&6.1&25.0&33.6&44.5\\
ADVENT \cite{vu2019advent}& 89.4&33.1&81.0&26.6&26.8&27.2&33.5&24.7&83.9&36.7&78.8&58.7&30.5&84.8&38.5&44.5&1.7&31.6&32.4&45.5\\
DISE \cite{chang2019all}& 91.5&47.5&82.5&31.3&25.6&33.0&33.7&25.8&82.7&28.8&82.7&62.4&30.8&85.2&27.7&34.5&6.4&25.2&24.4&45.4\\
BDL \cite{li2019bidirectional}& 91.0&44.7&84.2&\textbf{34.6}&27.6&30.2&36.0&36.0&85.0&\textbf{43.6}&83.0&58.6&31.6&83.3&35.3&\textbf{49.7}&3.3&28.8&35.6&48.5\\
SIBAN \cite{luo2019significance}& 88.5&35.4&79.5&26.3&24.3&28.5&32.5&18.3&81.2&40.0&76.5&58.1&25.8&82.6&30.3&34.4&3.4&21.6&21.5&42.6\\
SSF-DAN \cite{du2019ssf}& 90.3&38.9&81.7&24.8&22.9&30.5&37.0&21.2&84.8&38.8&76.9&58.8&30.7&85.7&30.6&38.1&5.9&28.3&36.9&45.4\\
MaxSquare \cite{chen2019domain}& 89.4&43.0&82.1&30.5&21.3&30.3&34.7&24.0&85.3&39.4&78.2&\textbf{63.0}&22.9&84.6&36.4&43.0&5.5&\textbf{34.7}&33.5&46.4\\
AdaptPatch \cite{tsai2019domain}& \textbf{92.3}&\textbf{51.9}&82.1&29.2&25.1&24.5&33.8&33.0&82.4&32.8&82.2&58.6&27.2&84.3&33.4&46.3&2.2&29.5&32.3&46.5\\
TGCF-SE \cite{choi2019self}& 88.9&29.9&83.9&31.4&28.1&28.1&35.5&25.5&84.7&38.4&83.1&60.2&26.7&\textbf{86.6}&\textbf{47.9}&49.1&8.1&25.8&24.4&46.6\\
IntraDA \cite{pan2020unsupervised}& 90.6&37.1&82.6&30.1&19.1&29.5&32.4&20.6&85.7&40.5&79.7&58.7&31.1&86.3&31.5&48.3&0.0&30.2&35.8&46.3\\
LSE \cite{subhani2020learning}& 90.2&40.0&83.5&31.9&26.4&32.6&38.7&37.5&81.0&34.2&84.6&61.6&\textbf{33.4}&82.5&32.8&45.9&6.7&29.1&30.6&47.5\\
\hline
Ours (NoAdapt) & 76.0&21.2&75.6&19.4&22.1&21.9&32.3&19.2&78.6&21.7&64.8&53.6&26.1&48.1&29.5&37.0&0.1&24.6&34.2&37.2\\
Ours (SE-GAN) &91.0&35.2&\textbf{85.1}&32.1&\textbf{28.4}&33.1&\textbf{41.9}&\textbf{42.6}&\textbf{86.2}&42.4&\textbf{85.1}&60.1&28.8&86.1&42.0&48.1&\textbf{16.7}&22.6&\textbf{44.2}&\textbf{50.1}\\
\bottomrule
\end{tabular}
}
\label{tab:gta5}
\end{table*}

\begin{table*}[t!]
\caption{Quantitative results of semantic segmentation by domain adaptation from Synthia to CityScapes. Best results in each column are highlighted in \textbf{bold}. The results of 16 common object categories are reported as mIoU and the results of 13 common categories (excluding wall, fence and pole) are reported as mIoU*.}
\centering
\resizebox{\textwidth}{!}{
\begin{tabular}{c|cccccccccccccccc|cc} %
\toprule
\multicolumn{19}{c}{Synthia$\to$CityScapes}\\
\hline
Methods&\rotatebox{90}{road}&\rotatebox{90}{sidewalk}&\rotatebox{90}{building}&\rotatebox{90}{wall}&\rotatebox{90}{fence}&\rotatebox{90}{pole}&\rotatebox{90}{light}&\rotatebox{90}{sign}&\rotatebox{90}{vegetation}&\rotatebox{90}{sky}&\rotatebox{90}{person}&\rotatebox{90}{rider}&\rotatebox{90}{car}&\rotatebox{90}{bus}&\rotatebox{90}{motocycle}&\rotatebox{90}{bicycle}&\rotatebox{90}{mIoU}&\rotatebox{90}{mIoU*}\\
\midrule
AdaptSeg \cite{adaptseg} &84.3&42.7&77.5&-&-&-&4.7&7.0&77.9&82.5&54.3&21.0&72.3&32.2&18.9&32.3&-&46.7\\
DCAN \cite{wu2018dcan}& 81.5&33.4&72.4&7.9&0.2&20.0&8.6&10.5&71.0&68.7&51.5&18.7&75.3&22.7&12.8&28.1&36.5&42.7\\
CLAN \cite{luo2019taking} &81.3&37.0&80.1&-&-&-&16.1&13.7&78.2&81.5&53.4&21.2&73.0&32.9&22.6&30.7&-&47.8\\
SWD \cite{lee2019sliced} &82.4&33.2&82.5&-&-&-&22.6&19.7&83.7&78.8&44.0&17.9&75.4&30.2&14.4&39.9&-&48.1\\
ADVENT \cite{vu2019advent} &85.6&42.2&79.7&8.7&0.4&25.9&5.4&8.1&80.4&84.1&57.9&23.8&73.3&36.4&14.2&33.0&41.2&48.0\\
DISE \cite{chang2019all} &\textbf{91.7}&\textbf{53.5}&77.1&2.5&0.2&27.1&6.2&7.6&78.4&81.2&55.8&19.2&82.3&30.3&17.1&34.3&41.5&48.8\\
BDL \cite{li2019bidirectional}& 86.0&46.7&80.3&-&-&-&14.1&11.6&79.2&81.3&54.1&\textbf{27.9}&73.7&\textbf{42.2}&25.7&45.3&-&51.4\\
SIBAN \cite{luo2019significance} &82.5&24.0&79.4&-&-&-&16.5&12.7&79.2&82.8&58.3&18.0&79.3&25.3&17.6&25.9&-&46.3\\
SSF-DAN \cite{du2019ssf}& 84.6&41.7&80.8&-&-&-&11.5&14.7&80.8&85.3&57.5&21.6&82.0&36.0&19.3&34.5&-&50.0\\
MaxSquare \cite{chen2019domain}& 82.9&40.7&80.3&10.2&0.8&25.8&12.8&18.2&82.5&82.2&53.1&18.0&79.0&31.4&10.4&35.6&41.4&48.2\\
AdaptPatch \cite{tsai2019domain}& 82.4&38.0&78.6&8.7&0.6&26.0&3.9&11.1&75.5&84.6&53.5&21.6&71.4&32.6&19.3&31.7&40.0&46.5\\
TGCF-SE \cite{choi2019self}& 88.9&46.2&81.1&13.2&1.2&29.0&25.1&21.2&73.7&56.7&57.3&23.1&83.9&38.4&16.2&46.4&43.8&50.6\\
DADA \cite{vu2019dada}& 89.2&44.8&81.4&6.8&0.3&26.2&8.6&11.1&81.8&84.0&54.7&19.3&79.7&40.7&14.0&38.8&42.6&49.8\\
IntraDA \cite{pan2020unsupervised}& 84.3&37.7&79.5&5.3&0.4&24.9&9.2&8.4&80.0&84.1&57.2&23.0&78.0&38.1&20.3&36.5&41.7&48.9\\
LSE \cite{subhani2020learning}&
82.9&43.1&78.1&9.3&0.6&28.2&9.1&14.4&77.0&83.5&58.1&25.9&71.9&38.0&\textbf{29.4}&31.2&42.6&49.4\\
\hline
Ours (NoAdapt)& 56.7&20.6&74.6&5.9&0.1&22.1&7.9&9.9&74.6&79.9&53.3&15.9&49.5&22.9&13.2&24.8&33.3&38.8\\
Ours (SE-GAN) &90.0&46.3&\textbf{83.7}&\textbf{20.0}&\textbf{1.9}&\textbf{33.9}&\textbf{34.9}&\textbf{24.6}&\textbf{85.5}&\textbf{86.3}&\textbf{59.0}&19.7&\textbf{85.1}&41.4&20.6&\textbf{50.1}&\textbf{48.9}&\textbf{55.9}\\
\bottomrule
\end{tabular}
}
\label{tab:synthia}
\end{table*}

\begin{figure*}[t]
\centering
\includegraphics[width=0.95\linewidth]{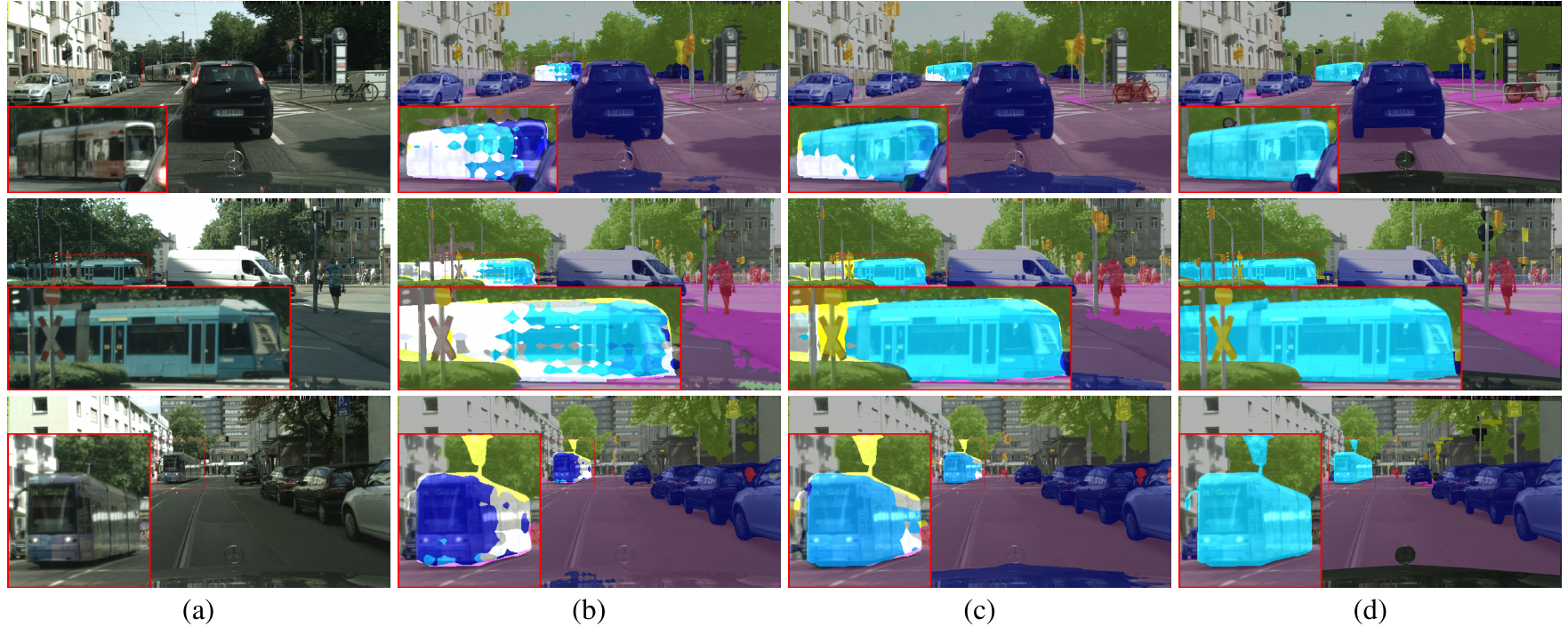}
\caption{Recognition results on the ``train'' category from GTA-5 to Cityscapes datasets. (a) Input images from Cityscapes dataset. (b) Adversarial Training. (c) SE-GAN. (d) Ground-truth annotations.}
\label{fig:trainmap}
\end{figure*}

\subsection{Experimental Results}

In this subsection, we evaluate the performance of the proposed framework against recent state-of-the-art methods. Note that all methods reported here adopt ResNet-101 as the backbone network to ensure fair comparisons. A brief introduction to these methods is given below.

\begin{itemize}
\item AdaptSeg \cite{adaptseg} adopts adversarial learning in the output space with a multi-level adversarial network.
\item DCAN \cite{wu2018dcan} performs channel-wise feature alignment to reduce domain shift at both pixel-level and feature-level.
\item CLAN \cite{luo2019taking} reduces the weight of the adversarial loss for category-level aligned features.
\item SWD \cite{lee2019sliced} utilizes the task-specific decision boundary and the Wasserstein metric for feature distribution alignment.
\item ADVENT \cite{vu2019advent} introduces an entropy-based adversarial training approach targeting both the entropy minimization objective and the structure adaptation.
\item DISE \cite{chang2019all} disentangles images into domain-invariant structure and domain-specific texture representations.
\item BDL \cite{li2019bidirectional} learns the image translation and the segmentation model alternatively by bidirectional learning.
\item SIBAN \cite{luo2019significance} enables a significance-aware feature purification before the adversarial adaptation.
\item SSF-DAN \cite{du2019ssf} designs a semantic-wise separable discriminator to independently adapt semantic features across the target and source domains.
\item MaxSquare \cite{chen2019domain} balances the gradient of well-classified target samples to prevent the training from being dominated by easy-to-transfer samples in the target domain.
\item AdaptPatch \cite{tsai2019domain} learns discriminative feature representations of patches in the source domain by discovering multiple modes of patch-wise output distribution through the construction of a clustered space.
\item TGCF-SE \cite{choi2019self} adopts a style transfer network for data augmentation in self-ensembling model.
\item DADA \cite{vu2019dada} conducts domain adaptation with the knowledge of dense depth in the source domain.
\item IntraDA \cite{pan2020unsupervised} reduces the inter-domain and intra-domain gap together by the self-training technique.
\item LSE \cite{subhani2020learning} exploits the scale-invariance property of the semantic segmentation model for self-supervised domain adaptation.
\end{itemize}

The quantitative results are presented in Table~\ref{tab:gta5} and Table~\ref{tab:synthia}. We adopt the intersection over union (IoU) per category and mean IoU (mIoU) over all the categories as the performance metrics. ``NoAdapt'' represents the baseline segmentation models directly trained on the labeled source-domain data without domain adaptation. It can be seen from Table~\ref{tab:gta5} and Table~\ref{tab:synthia} that the mIoUs of SE-GAN are about $13\%$ and $16\%$ higher than those of the NoAdapt baseline for GTA-5 to Cityscapes and Synthia to Cityscapes scenarios, respectively. Further, the proposed framework significantly outperforms recent state-of-the-art methods in both domain adaptation scenarios in most categories. These results verify the effectiveness of the proposed method for reducing the domain gaps between source and target domains. Some qualitative semantic segmentation results from GTA-5 to Cityscapes are presented in Fig. \ref{fig:segablation}.

We also find that the proposed SE-GAN shows superiority in addressing those ``hard examples''. Here, we take the ``train'' category for instance. Since the Cityscapes ``trains'' are more
visually similar to the GTA-5 ``buses'' instead of the GTA-5 ``trains'' \cite{zhang2019curriculum}, most methods yield relatively poor performance in this category (less than $10\%$ in the IoU metric as shown in Table \ref{tab:gta5}). By contrast, SE-GAN could achieve an IoU of about $16.7\%$ on ``train'', which outperforms the NoAdapt baseline by more than $16\%$. Some recognition results on the ``train'' category from GTA-5 to Cityscapes are presented in Fig. \ref{fig:trainmap}.



\subsection{Ablation Study}
We further evaluate how each module in SE-GAN influences domain adaptation performance. The quantitative ablation study results are demonstrated in Table \ref{tab:ablation}. Here, ``AT'' denotes the adversarial training, ``SE'' denotes the self-ensembling learning, ``Aug'' denotes data augmentation with style-transferred source-domain images, ``ST'' denotes the self-training with pseudo labels, and ``MST'' represents multi-scale testing. In both datasets, directly applying AT only leads to relatively low mIoUs, while combining both SE and Aug techniques can significantly improve the performance. ST also plays an important role in SE-GAN, especially for the Synthia dataset. Finally, with the help of MST strategy, the mIoU further increases, achieving state-of-the-art performance.

\begin{table}[t]
\caption{Performance contribution of each module in SE-GAN (reported in mIoU). Best results are highlighted in \textbf{bold}.}
\centering
\resizebox{\linewidth}{!}{
\begin{tabular}{cccccc|cc} %
\toprule
\textbf{Method}&\textbf{AT}&\textbf{SE}&\textbf{Aug}&\textbf{ST}&\textbf{MST}&\textbf{GTA-5}&\textbf{Synthia}\\
\midrule
\textit{Baseline}&&&&&&37.2&33.3\\
$+$AT&$\checkmark$&&&&&43.1&41.0\\
$+$SE&$\checkmark$&$\checkmark$&&&&46.0&42.5\\
$+$Aug&$\checkmark$&$\checkmark$&$\checkmark$&&&48.5&45.0\\
$+$ST&$\checkmark$&$\checkmark$&$\checkmark$&$\checkmark$&&49.5&47.4\\
$+$MST&$\checkmark$&$\checkmark$&$\checkmark$&$\checkmark$&$\checkmark$&\textbf{50.1}&\textbf{48.9}\\
\bottomrule
\end{tabular}
}
\label{tab:ablation}
\end{table}


One of the benefits of the proposed SE-GAN is the strong stability during training. As shown in Fig. \ref{fig:training_plot}, both AT and AT+Aug can hardly obtain a stable mIoU in the training phase (blue and green curves), since the instability of training is a common barrier shared by most adversarial training methods. By contrast, using the self-ensembling model as the generator can dramatically stabilize adversarial training (red and purple curves).

We further make a detailed comparison of per-class IoU gains between adversarial training and the proposed SE-GAN. As shown in Fig. \ref{fig:gain}, there exist negative transfers on the fence, vegetation, terrain, truck, and motorcycle in the adversarial training method (blue bars). By contrast, the proposed SE-GAN significantly mitigates this phenomenon (red bars). The negative transfer only happens in the motorcycle class for SE-GAN, while in the other remaining four challenging classes, SE-GAN achieves remarkable IoU gains at least over $5\%$. These results are in accord with our intuition that ensemble learning could stabilize adversarial training and reduce the risk of negative transfer.

\begin{figure}
\centering
\includegraphics[width=\linewidth]{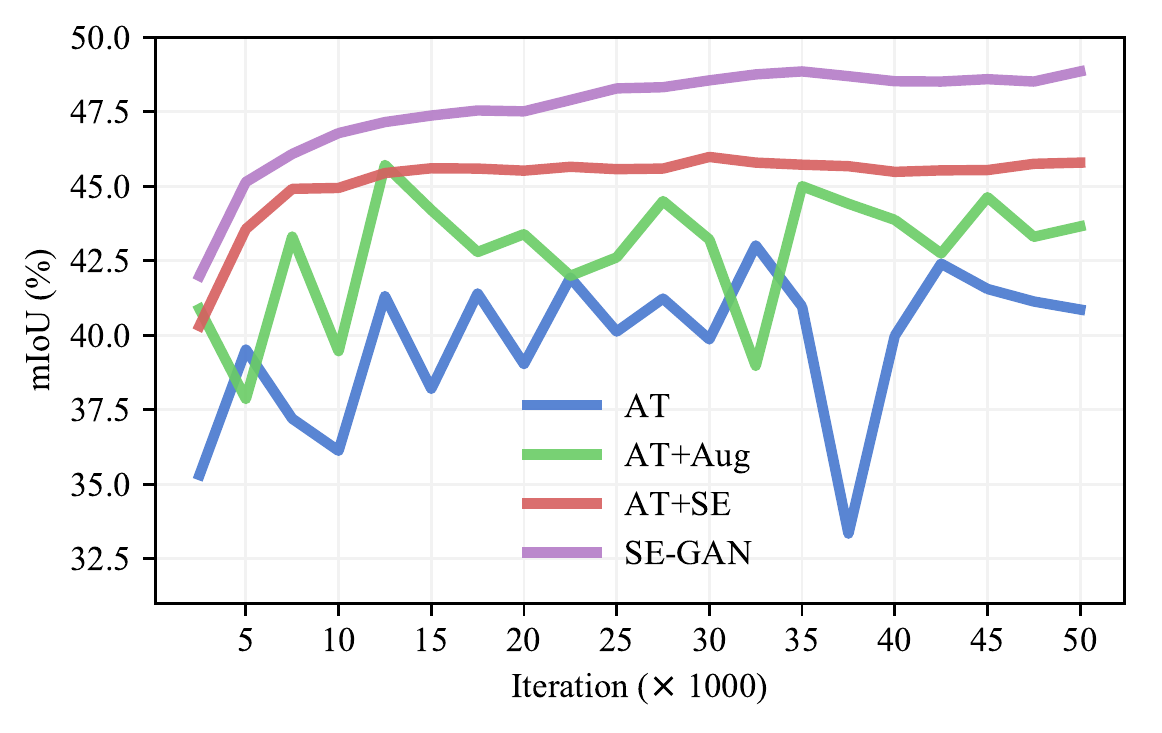}
\caption{The testing mIoUs of different adversarial training strategies at different iterations on the GTA-5 to Cityscapes scenario.}
\label{fig:training_plot}
\end{figure}

\begin{table}
\centering
\caption{Parameter analysis of the weighting factor $\lambda_{adv}$ in SE-GAN (without self-training). Best results are highlighted in \textbf{bold}.}
\begin{tabular}{cccccc}
\toprule
\multicolumn{6}{c}{GTA-5$\to$Cityscapes}\\
\hline
$\lambda_{adv}$&$1e-4$&$5e-4$&$1e-3$&$5e-3$&$1e-2$\\
\midrule
mIoU $\left(\%\right)$&47.8&48.2&\textbf{48.5}&48.1&47.5\\
\bottomrule
\end{tabular}
\label{tab:adv}
\end{table}

\subsection{Parameter Analysis}
In this subsection, we analyze how different values of the parameters in SE-GAN would influence the domain adaptation performance.

\textit{The adversarial weighting factor $\lambda_{adv}$ in (\ref{eq:4})}. Table \ref{tab:adv} shows that a smaller $\lambda_{adv}$ may not facilitate the adversarial training process too much, while a larger $\lambda_{adv}$ may bring about incorrect domain knowledge to the network. We empirically choose $\lambda_{adv}$ as $1e-3$ in the experiments.

\begin{figure}
\centering
\includegraphics[width=\linewidth]{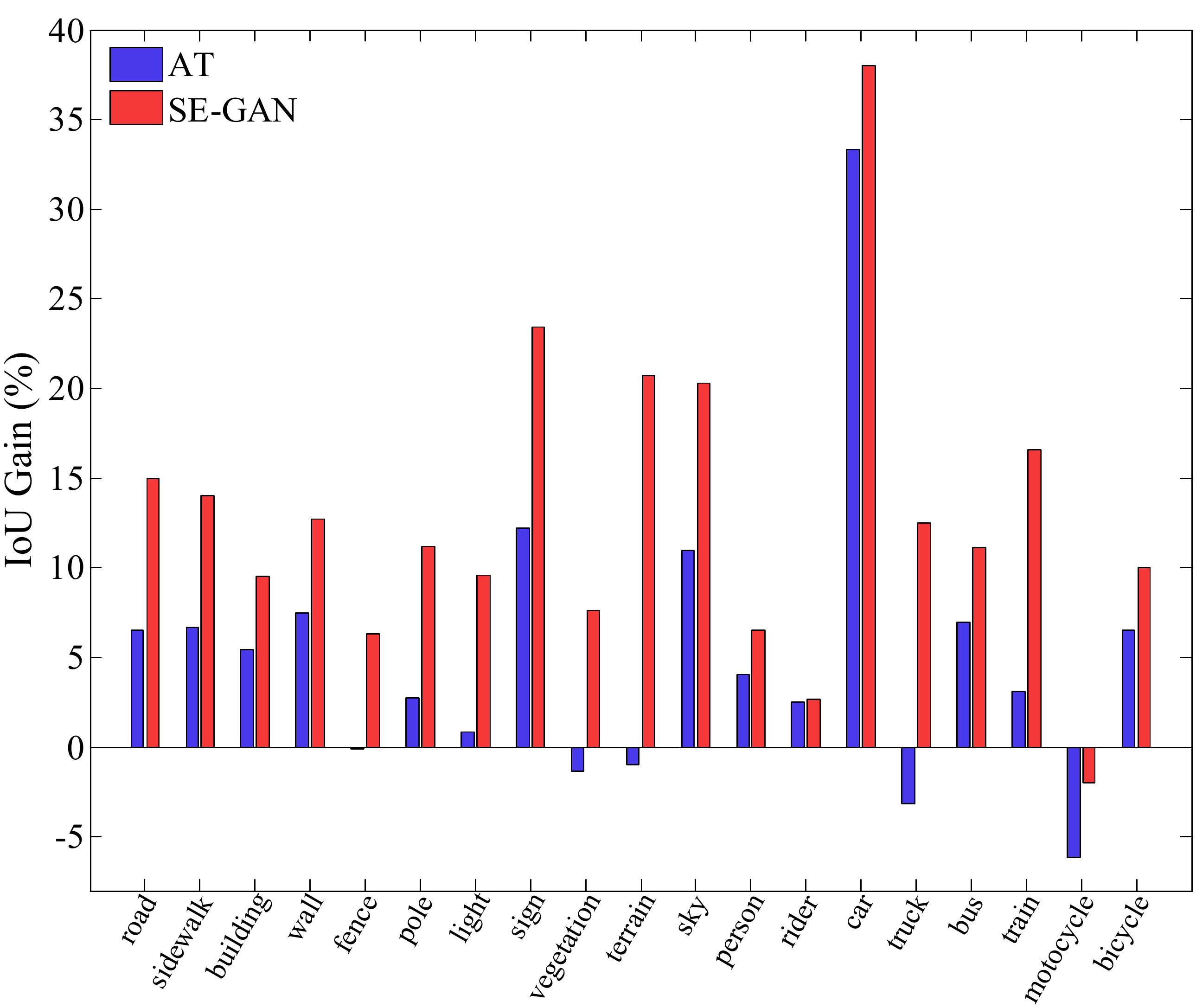}
\caption{Per-class IoU gains by AT and SE-GAN on the GTA-5 to Cityscapes scenario.}
\label{fig:gain}
\end{figure}

\textit{The self-ensembling weighting factor $\lambda_{con}$ in (\ref{eq:4})}. As shown in Table \ref{tab:con}, a too large $\lambda_{con}$ (i.e., $\lambda_{con}=10$) may bring about a too strong regularization for the training of student network, which is detrimental to the domain adaptation. When $\lambda_{con}=3$, SE-GAN can achieve the best mIoU accuracy.

\textit{The smoothing coefficient $\alpha$ in (\ref{eq:5})}. Table \ref{tab:alpha} shows that the ensemble model generated with a small value of $\alpha$ may not be strong enough. Generally, a larger $\alpha$ helps to improve ensemble learning. However, when $\alpha$ reaches $0.9999$, the performance drops significantly, which indicates the teacher network would get stuck in those incorrect historical models in this case. We empirically choose $\alpha$ as $0.999$ in the experiments.

\section{Conclusion and Future Direction}
This paper proposes a self-ensembling generative adversarial network (SE-GAN) for cross-domain semantic segmentation. In SE-GAN, a teacher network and a student network form a self-ensembling model for generating semantic segmentation maps, which together with a discriminator, form a GAN. While most of the existing adversarial training-based methods commonly suffer from the barrier of instability and fragility in the training phase, we find SE-GAN can significantly boost the performance of adversarial training and enhance the stability of the model. We evaluate the proposed method from both theoretical and empirical perspectives. First, we theoretically analyze the proposed method and provide an $\mathcal O\left(1/\sqrt{N}\right)$ generalization bound for SE-GAN. Then, extensive and systematic experiments verify that the proposed method can significantly stabilize adversarial training and outperform existing state-of-the-art approaches. Considering that the collection of high-quality pixel-level annotations is very time-consuming, SE-GAN has the potential to save a large amount of labeling effort.

SE-GAN entangles two promising yet distinct schemes, adversarial training and self-ensembling. It inherits advantages from both regimes and addresses each other's major critical shortcomings. Specifically, adversarial training enables self-ensembling networks to tackle domain shifts by employing well-designed loss functions, while the self-ensembling mechanism, in turn, enhances the stability of adversarial training. Most techniques that rely on deep neural networks can benefit from the exploitation of cross-domain data. These cross-domain methods ought to share similar technical barriers with cross-domain semantic segmentation. Therefore, {\it SE-GAN may be of independent interest in other cross-domain tasks.} We propose to explore it in future works.

\begin{table}
\centering
\caption{Parameter analysis of the weighting factor $\lambda_{con}$ in SE-GAN (without self-training). Best results are highlighted in \textbf{bold}.}
\begin{tabular}{cccccc}
\toprule
\multicolumn{6}{c}{GTA-5$\to$Cityscapes}\\
\hline
$\lambda_{con}$&0.1&0.3&1&3&10\\
\midrule
mIoU $\left(\%\right)$&47.3&47.9&48.1&\textbf{48.5}&46.1\\
\bottomrule
\end{tabular}
\label{tab:con}
\end{table}

\begin{table}
\centering
\caption{Parameter analysis of the smoothing coefficient $\alpha$ in SE-GAN (without self-training). Best results are highlighted in \textbf{bold}.}
\begin{tabular}{cccccc}
\toprule
\multicolumn{6}{c}{GTA-5$\to$Cityscapes}\\
\hline
$\alpha$&0.9&0.99&0.995&0.999&0.9999\\
\midrule
mIoU $\left(\%\right)$&47.1&47.3&47.9&\textbf{48.5}&44.6\\
\bottomrule
\end{tabular}
\label{tab:alpha}
\end{table}

\appendices
\section{}
This appendix first presents all the proofs omitted from the main text. Then, the style transfer model used in this study is described in detail.
\subsection{Proof of Theorem \ref{thm:cover_bound}}

This section provides detailed proof for Theorem \ref{thm:cover_bound}. We first recall two lemmas by Bartlett \textit{et al.} \cite{bartlett2017spectrally}. 

\begin{lemma}[cf. \cite{bartlett2017spectrally}, Lemma A.7]
\label{coverChainBoundGeneral}
Suppose there are $L$ weight matrices in a chain-like neural network. Let $(\varepsilon_{1}, \ldots, \varepsilon_{L})$ be given. Suppose the $L$ weight matrices $(A_{1}, \ldots, A_{L})$ lies in $\mathcal B_{1} \times \ldots \times \mathcal B_{L}$, where $\mathcal B_{i}$ is a ball centered at $0$ with the radius of $s_{i}$, i.e., $\mathcal B_{i} = \{A_{i}: \| A_{i} \| \le s_{i}\}$. Furthermore, suppose the input data matrix $X$ is restricted in a ball centred at $0$ with the radius of $B$, i.e., $\| X \| \le B$. Suppose $F$ is a hypothesis function computed by the neural network. If we define:
\begin{equation}
    \mathcal H = \{ F(X): A_{i} \in \mathcal B_{i}, A^{u,v,s}_{t} \in \mathcal B^{u,v,s}_{t} \} ~,
\end{equation}
where $i = 1, \ldots, L$, $(u,v,s)\in I_{V}$, and $t \in \{1, \ldots, L^{u,v,s}\}$. Let $\varepsilon = \sum_{j = 1}^{L}\varepsilon_{j}\rho_{j}\prod_{l = j+1}^{L}\rho_{l}s_{l}$. Then we have the following inequality:
\begin{align}
\label{formulaCoverChainBoundGeneral}
    \mathcal N(\mathcal H) \le \prod_{i=1}^{L} \sup_{\mathbf A_{i-1} \in \bm{\mathcal B}_{i-1}} \mathcal N_{i} ~,
\end{align}
where $\mathbf A_{i-1} = (A_{1}, \ldots, A_{i-1})$, $\bm{\mathcal B}_{i-1} = \mathcal B_{1} \times \ldots \times \mathcal B_{i-1}$, and
\begin{equation}
    \mathcal N_{i}  = \mathcal N \left( \left\{ A_{i}F_{\mathbf A_{i-1}}(X): A_{i} \in \mathcal B_{i} \right\} \varepsilon_{i}, \| \cdot \| \right) ~.
\end{equation}
Here, the radius of each covers are respectively,
\begin{equation}
\label{eq:radius}
    \varepsilon_i = \frac{\alpha_i \varepsilon}{\rho_i\prod_{j>i}\rho_js_j},
\end{equation}
where
\begin{gather}
\label{eq:radius_par1}
    \alpha_i=\frac{1}{\bar\alpha} \left(\frac{b_i}{s_i} \right)^{2/3},\\
\label{eq:radius_par2}
    \bar \alpha = \sum_{j=1}^L \left(\frac{b_j}{s_j}\right)^{2/3}.
\end{gather}
\end{lemma}

\begin{lemma}[cf. \cite{bartlett2017spectrally}, Lemma 3.2]
\label{matrixCover}
    Let conjugate exponents $(p, q)$ and $(r, s)$ be given with $p \le 2$, as well as positive reals $(a, b, \varepsilon)$ and positive integer $m$. Let matrix $X \in \mathbb R^{n \times d}$ be given with $\| X \|_{p} \le b$. Let $\mathcal H_{A}$ denote the family of matrices obtained by evaluating $X$ with all choices of matrix $A$:
    \begin{equation}
        \mathcal H_{A} \triangleq \left\{ XA | A \in \mathbb R^{d \times m}, \|A\|_{q, s} \le a \right\} ~.
    \end{equation}
    Then
    \begin{equation}
    \label{matrixCovBound}
        \log \mathcal N \left( \mathcal H_{A}, \varepsilon, \| \cdot \|_{2} \right) \le \ceil*{\frac{a^{2}b^{2}m^{2/r}}{\varepsilon^{2}}} \log(2dm).
    \end{equation}
\end{lemma}
This covering bound constrains the hypothesis complexity contributed by a single weight matrix.

\begin{proof}[Proof of Theorem \ref{thm:cover_bound}]

Suppose the hypothesis spaces of the output functions $F_{(A_{1}, \ldots, A_{i-1})}$ of the weight matrices $A_{i}, i = 1, \ldots, 5$ are respectively $\mathcal H_{i}, i = 1, \ldots, 5$. From Lemma \ref{coverChainBoundGeneral}, we can directly get the following inequality,
\begin{align}
    & \log \mathcal N(\mathcal F|S) \nonumber\\
    \le & \log \left( \prod_{i=1}^{5} \sup_{\mathbf A_{i-1} \in \bm{\mathcal B}_{i-1}} \mathcal N_{i} \right) \nonumber\\
    \le & \sum_{i = 1}^{5} \log \left( \sup_{\substack{(A_{1}, \ldots, A_{i-1}) \\ \forall j < i, A_{j} \in B_{j}}} \mathcal N \left( \left\{ A_{i} F_{(A_{1}, \ldots, A_{i-1})} \right\}, \varepsilon_{i}, \| \cdot \|_{2} \right) \right) ~.
\end{align}
Employ eq. (\ref{matrixCovBound}), we can get the following inequality,
\begin{align}
    \log \mathcal N(\mathcal F|S) \le \sum_{i = 1}^{5} \frac{b_{i}^{2} \| F_{(A_{1}, \ldots, A_{i-1})}(X) \|^{2}_{\sigma}}{\varepsilon_{i}^{2}} \log \left( 2W^{2} \right) ~. 
\end{align}
Meanwhile,
\begin{align}
    &\| F_{(A_{1}, \ldots, A_{i-1})}(X) \|^{2}_{\sigma} \\
=& \| \sigma_{i-1} (A_{i-1} F_{(A_{1}, \ldots, A_{i-2})}(X)) - \sigma_{i-1}(0) \|_2 \nonumber\\
    \le & \| \sigma_{i-1} \| \| A_{i-1} F_{(A_{1}, \ldots, A_{i-2})}(X) - 0 \|_2 \nonumber\\
    \le & \rho_{i-1} \| A_{i-1} \|_{\sigma} \| F_{(A_{1}, \ldots, A_{i-2})}(X) \|_2 \nonumber\\
    \le & \rho_{i-1} s_{i-1} \| F_{(A_{1}, \ldots, A_{i-2})}(X) \|_2.
\end{align}
Therefore,
\begin{equation}
    \| F_{(A_{1}, \ldots, A_{i-1})}(X) \|^{2}_{\sigma} \le \|X\|^{2} \prod_{j = 1}^{i-1} s_{i}^{2} \rho_{i}^{2}.
\end{equation}

Suppose the covering number radius of the final output hypothesis space is $\varepsilon$, then we can get the formulation of each covering number radius $\varepsilon_{i}$ throughout the neural network in terms of the radius $\varepsilon$. Specifically, motivated by the proof given in \cite{bartlett2017spectrally}, we can get the following equations:
\begin{gather}
    \varepsilon_{i+1} = \rho_{i} s_{i+1} \varepsilon_{i} ~.
\end{gather}
Then,
\begin{gather}
    \varepsilon_{5} = \rho_{1} \prod_{i = 2}^{4} s_{i}\rho_{i} s_{5} \epsilon_{1} ~,\\
    \varepsilon = \rho_{1} \prod_{i = 2}^{5} s_{i}\rho_{i} \epsilon_{1} ~.
\end{gather}
Therefore,
\begin{equation}
    \varepsilon_{i} = \frac{\rho_{i}\prod_{j = 1}^{i-1} s_{j}\rho_{j}}{\prod_{j = 1}^{5} s_{j}\rho_{j}} \varepsilon ~.
\end{equation}
Therefore,
\begin{align}
    &\log \mathcal N \left(\mathcal F|_{S}, \varepsilon, \| \cdot \|_{2} \right) \\
    &\le \frac{\log\left( 2W^{2} \right) \| X \|_{2}^{2}}{\varepsilon^{2}} \left( \prod_{i = 1}^{5} s_{i} \rho_{i} \right)^{2} \sum_{i = 1}^{5} \frac{b_{i}^{2}}{s_{i}^{2}} ~,
\end{align}
which is exactly eq. (\ref{eq:cover_bound}) of Theorem \ref{thm:cover_bound}.

The proof is completed.
\end{proof}

\begin{figure*}
\centering
\includegraphics[width=0.85\linewidth]{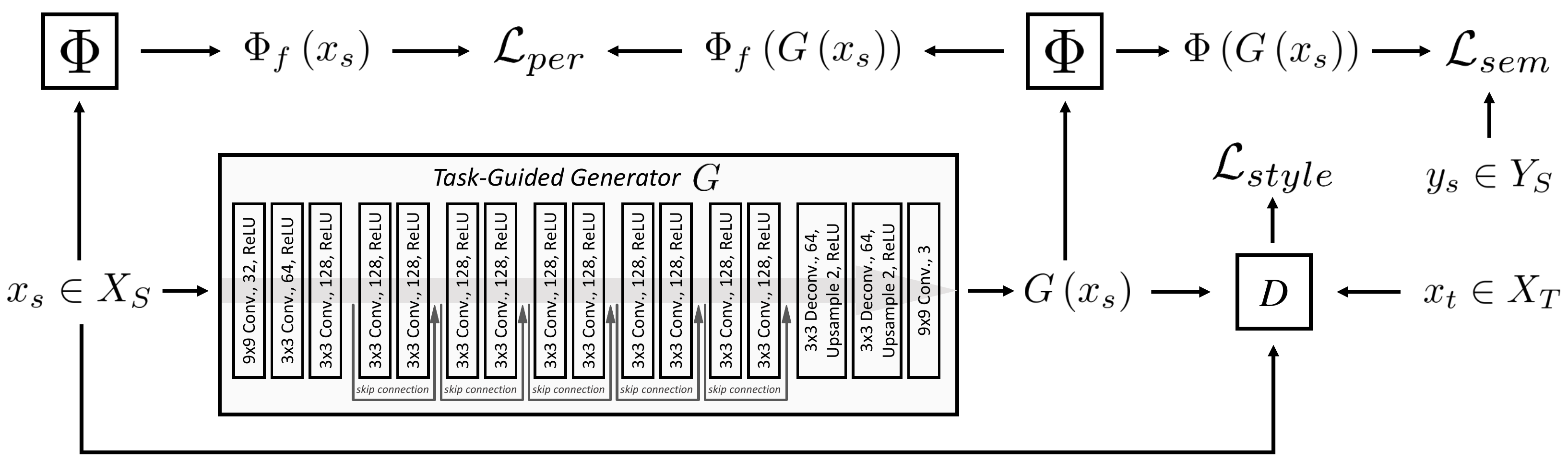}
\caption{An illustration of the proposed task-guided style transfer network (TGSTN).}
\label{fig:tgstn}
\end{figure*}

\subsection{Proof of Theorem \ref{thm:generalization_bound}}
\label{sec:generalization_bound}

This section provides a detailed proof for Theorem \ref{thm:generalization_bound}. We first recall a recent lemma generally addressing the generalization ability of GAN and a classic lemma in statistical learning theory.

\begin{lemma}[cf. \cite{zhang2017discrimination}, p. 8, Theorem 3.1]
\label{lemma:generalization_bound_gan}
    Assume that the discriminator set $F$ is even, i.e., $f \in \mathcal F$ implies $-f \in \mathcal F$, and that all discriminators are bounded by $\Delta$, i.e., $\| f \|_{\infty} \le \Delta$ for any $f \in \mathcal F$. Let $\hat \mu_{N}$ be an empirical measure of an independent and identical (i.i.d.) sample of size $N$ drawn from a distribution $\mu$. Assume $\nu_{N} \in \mathcal G$ satisfies
    \begin{equation}
        d_{\mathcal F} (\hat \mu_{N}, \nu_{N}) \le \inf_{\nu \in \mathcal G} d_{\mathcal F} (\hat \mu_{N}, \nu) + \phi ~.
    \end{equation}
    Then with probability at least $1 - \delta$, we have
    \begin{equation}
    \label{eq:generalization_bound_gan}
        d_{\mathcal F}(\mu, \nu_{N}) - \inf_{\nu \in \mathcal G} d_{\mathcal F}(\mu, \nu) \le 2 \mathfrak R^{(\mu)}_{N}(\mathcal F) + 2 \Delta \sqrt{\frac{2 \log (\frac{1}{\delta})}{N}} + \phi ~.
    \end{equation}
\end{lemma}

Computing the empirical Rademacher complexity of neural networks could be extremely difficult and thus still remains an open problem. Fortunately, the empirical Rademacher complexity can be upper bounded by the corresponding $\varepsilon$-covering number $N(\mathcal F, \varepsilon, \| \cdot \|_{2})$ as the following lemma states.
\begin{lemma}[cf. \cite{bartlett2017spectrally}, Lemma A.5]
\label{covNumBound}
    Suppose $\bm 0 \in \mathcal H$ and all conditions in Lemma \ref{lemma:generalization_bound_gan} hold. Then
    \begin{align}
    \label{formulaCovNumBound}
        \mathfrak R^{(\mu)}_{N}(\mathcal F) \le \inf_{\alpha > 0} \left( \frac{4\alpha}{\sqrt{n}} + \frac{12}{n} \int_{\alpha}^{\sqrt n} \sqrt{\log \mathcal N(l \circ \mathcal H, \varepsilon, \| \cdot \|_{2})} \text{d}\varepsilon \right) ~.
    \end{align}
\end{lemma}

\begin{proof}[Proof of Theorem \ref{thm:generalization_bound}]
    Apply Lemma \ref{covNumBound} directly to Theorem \ref{thm:cover_bound}, we can get the following equation
    \begin{align}
    \label{eq:rademarcher_complexity}
    \mathfrak R^{(\mu)}_{N} \le & \inf_{\alpha > 0} \left( \frac{4\alpha}{\sqrt{n}} + \frac{12}{n} \int_{\alpha}^{\sqrt n} \sqrt{\log \mathcal N(\mathcal H_{\lambda}|_{D}, \varepsilon, \| \cdot |_{2})} \text{d}\varepsilon \right) \nonumber\\
    \le & \inf_{\alpha > 0} \left( \frac{4\alpha}{\sqrt{n}} + \frac{12}{n} \int_{\alpha}^{\sqrt n} \frac{R}{\varepsilon} \text{d}\varepsilon \right) \nonumber\\
    \le & \inf_{\alpha > 0} \left[ \frac{4\alpha}{\sqrt{n}} + \frac{12}{n} \sqrt{R} \log \left( \frac{\sqrt{n}}{\alpha} \right) \right] ~.
    \end{align}
Apparently, the infinimum is reached uniquely at $\alpha = 3\sqrt{\frac{R}{n}}$ and the infinitum is as follows,
    \begin{align}
    \mathfrak R^{(\mu)}_{N} \le & \frac{12R}{N} \left[ 1 + \log \left( \frac{N}{3R} \right) \right] ~.
    \end{align}

    Apply eq. (\ref{eq:rademarcher_complexity}) to eq. (\ref{eq:generalization_bound_gan}) of Lemma \ref{lemma:generalization_bound_gan}, we can directly get the fowling equation,
    \begin{align}
        &d_{\mathcal F}(\mu, \nu_{N}) - \inf_{\nu \in \mathcal G} d_{\mathcal F}(\mu, \nu) \\
        &\le \frac{24R}{N} \left( 1 + \log \frac{N}{3R} \right) + 2 \Delta \sqrt{\frac{2 \log (\frac{1}{\delta})}{N}} + \phi ~,
    \end{align}
    which is exactly eq. (\ref{eq:generalization_bound}).
    The proof is completed.
\end{proof}

\begin{figure*}
\centering
\includegraphics[width=0.85\linewidth]{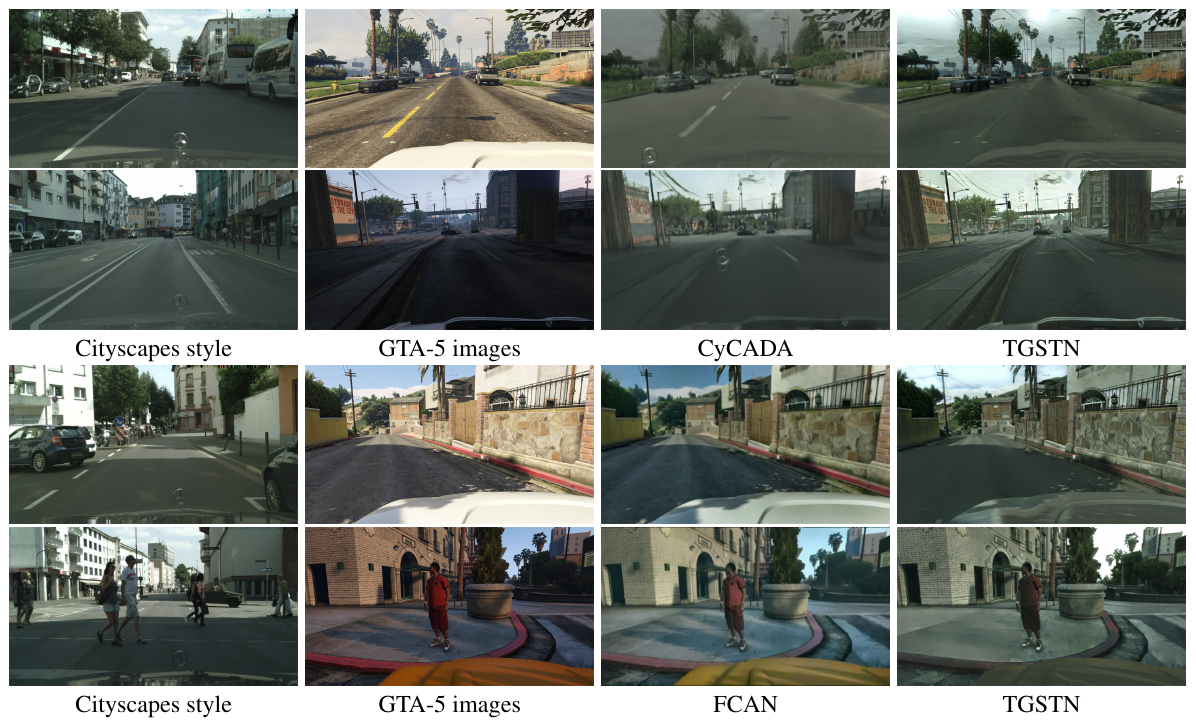}
\caption{Image style transfer from GTA-5 to CityScapes with different approaches. First column: Cityscapes images. Second column: Original GTA-5 images. Third column: Image translation by CyCADA \cite{hoffman2018cycada}, and FCAN \cite{zhang2018fully}. Fourth column: Image translation by the proposed TGSTN.}
\label{fig:imagetrans}
\end{figure*}

\subsection{Task-guided Style Transfer Network}
To address the visual appearance shift between the source and target domains, we propose a task-guided style transfer network (TGSTN) in this study. For each input source-domain image, we use TGSTN to transfer its visual style to be consistent with that of the target domain, while simultaneously retaining its semantic information.

The TGSTN architecture is shown in Fig. \ref{fig:tgstn}. Let $S$ and $T$ represent the source domain and target domain, respectively. Given a set of source domain images $X_{S}$, the corresponding labels $Y_{S}$, and target domain images $X_{T}$ (without annotations), TGSTN aims to learn a generative model $G$ that can transfer the visual style of $X_{S}$ in source domain $S$ to be consistent with that of $T$ while retaining the semantic layout of image $X_{S}$. To transfer the image style, we utilize a discriminator $D$ to distinguish the real target domain images $x_{t}\in X_{T}$ from the original source domain images $x_{s}\in X_{S}$ and the transferred images $G\left(x_{s}\right)$ according to the image style \cite{radford2015unsupervised}. Thus, the discriminator $D$ and the generator $G$ together form an adversarial learning model. The corresponding style adversarial loss $\mathcal{L}_{style}$ can be formulated as:
\begin{align}
    &\mathcal{L}_{style}\left(G,D\right)=\mathbb{E}_{x_t\sim X_T}\left[\log D\left(x_t\right)\right]+\nonumber\\
    &~ ~ ~ ~\mathbb{E}_{x_s\sim X_S}\left[\log \left(1-D\left(x_s\right)\right)+\log \left(1-D\left(G\left(x_s\right)\right)\right)\right].
\label{eq:loss_style}
\end{align}
Similar to previous methods \cite{zhu2017unpaired,ledig2017photo}, we optimize (\ref{eq:loss_style}) through a minimax two-player game. Specifically, we alternately optimize $\mathop{\min}_{G}\mathcal{L}_{style}$ and $\mathop{\max}_{D}\mathcal{L}_{style}$.

Although (\ref{eq:loss_style}) ensures that the generated sample $G\left(x_s\right)$ may share a similar distribution with $X_T$, it fails to maintain semantic information contained in the original sample $x_s$. In order to better preserve the content in $x_s$, we further adopt the semantic consistency loss and perceptual loss to assist the training of $G$.

Let $\Phi$ be a segmentation network pre-trained on $S$ with fixed weights and $\Phi\left(\cdot\right)$ be the segmentation output of $\Phi$. The semantic consistency loss $\mathcal{L}_{sem}$ can be defined as:
\begin{equation}
    \mathcal{L}_{sem}\left(G\right)=\frac{-1}{K}\sum_{k=1}^{K}\sum_{c=1}^{C}y_s^{\left(k,c\right)}
    \log \left(\sigma\left(\Phi\left(G\left(x_s\right)\right)\right)^{\left(k,c\right)}\right),
\label{eq:loss_sem}
\end{equation}
where $K$ and $C$ denote the number of pixels in the image, and the number of categories in the segmentation task, respectively. $\sigma$ denotes the softmax function. This constraint can encourage the transferred image $G\left(x_s\right)$ to possess the same semantic information as the original label $y_s$.

Let $\Phi_f\left(\cdot\right)$ be the output of the last feature extraction layer before the classification layer in $\Phi$. Then, the perceptual loss $\mathcal{L}_{per}$ can be formulated as:
\begin{equation}
    \mathcal{L}_{per}\left(G\right)=\frac{1}{K_f}\sum_{k=1}^{K_f}\left\|\Phi_f\left(G\left(x_s\right)\right)^k-\Phi_f\left(x_s\right)^k\right\|^2,
\label{eq:loss_per}
\end{equation}
where $K_f$ represents the number of pixels in the feature map. With the constraint in (\ref{eq:loss_per}), the transferred image $G\left(x_s\right)$ is encouraged to share an identical high-level feature representation with the original $x_s$ \cite{johnson2016perceptual}.

The full objective function for training TGSTN can therefore be formulated as:
\begin{equation}
    \mathop{\min}_{G}\mathop{\max}_{D}\mathcal{L}_{style}\left(G,D\right)+\lambda_{sem}\mathcal{L}_{sem}\left(G\right)+\lambda_{per}\mathcal{L}_{per}\left(G\right),
\label{eq:loss_tgstn}
\end{equation}
where $\lambda_{sem}$ and $\lambda_{per}$ are two weighting factors.

In the practical implementation of the proposed TGSTN, we employ DeepLab-v2 \cite{deeplab} with the VGG-16 \cite{simonyan2014very} backbone model pre-trained on ImageNet \cite{imagenet} as the segmentation network $\Phi$, which is further fine-tuned with source domain data. The architecture of the discriminator $D$ is the same with $D_{lab}$, as described in Section V-B. The Adam optimizer \cite{kingma2014adam} with a weight decay of $5e-5$ is utilized to train TGSTN, where the learning rate is set at $5e-4$ and $5e-5$ for $G$ and $D$, respectively. Similar to \cite{adaptseg}, we use the ``poly'' learning rate decay policy with a power of $0.9$. The training epochs are set to $5$. Each mini-batch consists of 2 source-domain images and 2 target-domain images. We resize all the images to $1024\times512$ pixels during training. $\lambda_{sem}$ and $\lambda_{per}$ in (\ref{eq:loss_tgstn}) are set at $10$ and $1$, respectively.

Fig. \ref{fig:imagetrans} shows the visual comparison of the transferred images by TGSTN from GTA-5 to Cityscapes with two state-of-the-art generative methods: CyCADA \cite{hoffman2018cycada}, and FCAN \cite{zhang2018fully}. We use the images presented in their original papers for comparison. Compared to CyCADA, TGSTN can better maintain the detailed semantic information from the source domain. Take the results in the second row of Fig. \ref{fig:imagetrans} for example. While CyCADA generates an unreal auto logo of Benz (which appears a lot in the target domain) in the center area of the image, TGSTN avoids this forged texture. Compared to the transferred images of FCAN, it can also be observed that TGSTN can generate a more similar visual style with the target domain while keeping more detailed semantic information from the source domain.

%
%

\ifCLASSOPTIONcaptionsoff
  \newpage
\fi



%

\bibliographystyle{IEEEtran}

\bibliography{SEGAN}

%

%




\end{document}